%% file: main.tex
\documentclass{article}

\PassOptionsToPackage{numbers, compress, sort}{natbib}
\PassOptionsToPackage{dvipsnames}{xcolor}

\usepackage[preprint]{neurips_2025}

\usepackage[utf8]{inputenc} 
\usepackage[T1]{fontenc}    
\usepackage{hyperref}       
\usepackage{url}            
\usepackage{booktabs}       
\usepackage{amsfonts}       
\usepackage{nicefrac}       
\usepackage{microtype}      
\usepackage{xcolor}         
\usepackage{soul}
\usepackage{wrapfig}
\usepackage{enumitem}
\usepackage[title]{appendix}
\usepackage{multirow}

\hypersetup{
  colorlinks,
  linkcolor={blue!50!black},
  citecolor={blue!50!black},
  urlcolor={blue!50!black}
}

\usepackage{amsmath}
\usepackage{amssymb}
\usepackage{mathtools}
\usepackage{amsthm}

\usepackage{algorithm}
\usepackage{algpseudocode}

\input{math_commands}

\usepackage[capitalize,noabbrev]{cleveref}

\newtheorem{theorem}{Theorem}

\theoremstyle{definition}

\usepackage{tikz}
\usepackage{tikzscale}
\usetikzlibrary{positioning}

\title{Simplifying Bayesian Optimization Via \\ In-Context Direct Optimum Sampling}

\newcommand{\authorinfo}{
Gustavo Sutter$^{1,2}$ \quad Mohammed Abdulrahman$^{1,2}$ \quad Hao Wang$^{1}$ \vspace{0.5em}\\ \bf  Sriram Ganapathi Subramanian$^{2}$ \quad Marc St-Aubin$^{3}$ \quad Sharon O'Sullivan$^{3}$ \quad Lawrence Wan$^{3}$\vspace{0.5em}\\ \bf Luis Ricardez-Sandoval$^{1}$ \quad Pascal Poupart$^{1,2}$ \quad Agustinus Kristiadi$^{2}$\vspace{0.5em}\\
	$^1$ University of Waterloo\\
	$^2$ Vector Institute\\
	$^3$ BMO, Technology \& Operations\\
    \vspace{-2em}
}

\author{\authorinfo}

\begin{document}

\maketitle

\begin{abstract}
  The optimization of expensive black-box functions is ubiquitous in science and engineering.
  A common solution to this problem is Bayesian optimization (BO), which is generally comprised of two components: (i) a surrogate model and (ii) an acquisition function,
  which generally require expensive re-training and optimization steps at each iteration, respectively.
  Although recent work enabled in-context surrogate models that do not require re-training, virtually all existing BO methods still require acquisition function maximization to select the next observation, which introduces many knobs to tune, such as Monte Carlo samplers and multi-start optimizers.
  In this work, we propose a completely in-context, zero-shot solution for BO that does not require surrogate fitting or acquisition function optimization.
  This is done by using a pre-trained deep generative model to directly sample from the posterior over the optimum point.
  We show that this process is equivalent to Thompson sampling and demonstrate the capabilities and cost-effectiveness of our foundation model on a suite of real-world benchmarks.
  We achieve an efficiency gain of more than \( 35\times \) in terms of wall-clock time when compared with Gaussian process-based BO, enabling efficient parallel and distributed BO, e.g., for high-throughput optimization.
\end{abstract}

\input{contents/01_intro}
\input{contents/02_background}

\input{contents/03_method}
\input{contents/04_related}
\input{contents/05_experiments}
\input{contents/06_results}
\input{contents/08_conclusion}

{
\small
\bibliography{main.bib}
\bibliographystyle{plainnat}
}


\clearpage
\begin{appendices}
    \input{contents/09_app_data_gen}
    \input{contents/10_app_proof}

    \input{contents/11_app_implementation}
    \input{contents/12_app_olympus}
\end{appendices}

\end{document}

%% file: math_commands.tex
\usepackage{amsmath,amsfonts,bm}

\def\1{\bm{1}}
\newcommand{\train}{\mathcal{D}}

\def\vzero{{\bm{0}}}

\def\vbeta{{\bm{\beta}}}

\def\vb{{\bm{b}}}
\def\vc{{\bm{c}}}

\def\vl{{\bm{l}}}

\def\vx{{\bm{x}}}

\def\vz{{\bm{z}}}

\def\mI{{\bm{I}}}

\def\mW{{\bm{W}}}

\DeclareMathAlphabet{\mathsfit}{\encodingdefault}{\sfdefault}{m}{sl}
\SetMathAlphabet{\mathsfit}{bold}{\encodingdefault}{\sfdefault}{bx}{n}

\def\gE{{\mathcal{E}}}

\def\gG{{\mathcal{G}}}

\def\gL{{\mathcal{L}}}

\def\gN{{\mathcal{N}}}

\def\gU{{\mathcal{U}}}

\def\gX{{\mathcal{X}}}

\def\sR{{\mathbb{R}}}

\newcommand{\E}{\mathbb{E}}

\newcommand{\R}{\mathbb{R}}

\DeclareMathOperator*{\argmax}{arg\,max}

\newcommand{\D}{\mathcal{D}}

%% file: contents/01_intro.tex
\section{Introduction}
\label{sec:intro}

\begin{wrapfigure}[15]{r}{0.4\textwidth}
    \centering
    \vspace{-1.75em}
    \includegraphics[width=\linewidth]{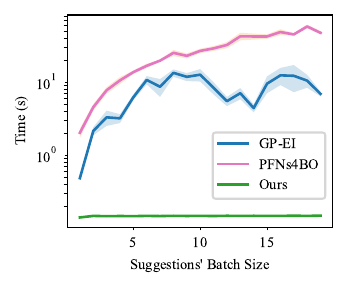}
    \vspace{-1.75em}
    \caption{
        Performing BO in-context with our method enables batched BO with a large batch size (almost) for free.
    }
    \label{fig:speed}
\end{wrapfigure}

Many problems in chemistry~\citep{griffiths2020constrained,greenaway2023alchemist}, biology~\citep{romero2013navigating,ruberg2023bayesdrug} and computer science~\citep{snoek2012boml,feurer2022auto} rely on optimizing an expensive black-box function.
Often, input domains are prohibitively large, and objective function evaluation needs laboratory experiments or computation-intensive simulation \citep{self-driving-lab2024tom}.
This necessitates specific black-box optimization methods that are data-efficient.
Ideally, those methods should be able to find the optimal value by leveraging pre-existing data and smart exploration strategies that query the objective function as little as possible.

An effective method for this type of problem is Bayesian optimization \citep[BO;][]{Mockus1975,garnett_bayesoptbook_2023}.
This family of methods works by sequentially recommending the next point \(\vx_{t+1}\) in a bid to maximize the target black-box function in as few steps as possible, under the guidance of a probabilistic surrogate model \(p(f \mid \D_t)\) trained on previous data points \(\D_t\).
Gaussian processes \citep[GPs;][]{Rasmussen_Williams_gp_2006} are the \emph{de facto} surrogate model for BO due to their tractable posterior inference, being supported by popular libraries and frameworks \citep{Balandat2019BoTorchPB}.
However, traditional GP scales cubically with the number of training points, assumes that targets are jointly Gaussian, and often uses stationary kernels.
Alternatively, one can use Bayesian neural networks, which excel in representing high-dimensional data and deal well with non-stationary objectives \citep{kristiadi_promises_2023,kristiadi_sober_2024}.
However, it is unclear how to pick the most suitable surrogate model for the problem at hand \citep{li_study_2023}.
Moreover, expensive re-training must be done at each BO step, incurring an additional computation overhead and more hyperparameters to tune.

\begin{figure*}[t]
  \centering
  \includegraphics[width=.9\textwidth]{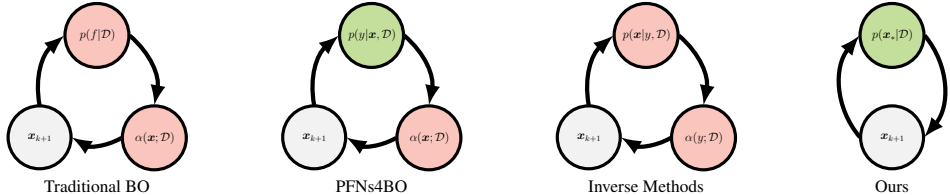}
\
  \caption{
    Optimization loops of different black-box optimizers: traditional BO, PFNs4BO~\citep{muller_pfns4bo_2023}, inverse methods such as Diff-BBO~\citep{wu_diff-bbo_2024} and DiBO~\citep{mittal2025amortizedincontextbayesianposterior}, and FIBO (ours).
      {\color{Salmon}Red nodes} represent stages that require model fitting or optimization (undesirable), while {\color{LimeGreen} green nodes} represent in-context learning operations (desirable).
    Fewer nodes are desirable since it implies fewer moving parts.
  }
  
  \label{fig:loop}
\end{figure*}

To alleviate these issues, recent work proposed in-context learning for surrogate modelling \citep{muller_pfns4bo_2023}.
These works reduce the cost of BO by using a surrogate that does not need to be re-trained as new data points are collected.
Instead, a foundation model is pretrained on a large number of functions coming from a predetermined prior distribution.
This strategy of approximate inference allows for a flexible prior without requiring traditionally expensive Bayesian inference methods to compute the approximate posterior.
An alternative approach is inverse models that generate input points given a target function value \citep{kumarlevine2020model}.
Recent work used conditional diffusion models to tackle the problem, leveraging their power to model complex distributions over high-dimensional spaces~\citep{krishnamoorthy_diffusion_2023,wu_diff-bbo_2024}.
Inverse methods, however, may require datasets with tens of thousands of data points to fit the generative model used in the objective optimization loop.

Nevertheless, previous work only aims to solve the surrogate-modeling question that we have raised---the other, equally important parts of  BO have so far been neglected.
Indeed, in addition to surrogate modeling, one must also pick a suitable acquisition function \(\alpha(\,\cdot\, \mid \D): \mathcal{X} \to \R\) that depends on the posterior, such as the expected improvement \citep{Jones1998EfficientGO} or upper confidence bound \citep{Auer2003UsingCB} functions, and optimize it to obtain the proposed point \(\vx_{t+1} = \argmax_{\vx \in \mathcal{X}} \alpha(\vx \mid \D_t)\). 
This optimization process is known to be hard and/or expensive \citep{ament_unexpected_2024}.

In this work, we aim to simplify Bayesian optimization by directly sampling from the posterior \(p(\vx_* \mid \D_t)\) of the \emph{optimum} \(\vx_*\) in an in-context way.
That is, we skip \emph{both} the surrogate modeling \emph{and} acquisition function maximization, reducing the number of moving parts in BO to its minimum (\cref{fig:loop}).
To this end, we pretrain a generative model on the pairs \((\D, \vx_*)\) of a context \(\D := \{ (\vx_i, f(\vx_i)) \}_{i}\) and the optimal point \(\vx_* = \argmax_{x \in \mathcal{X}} f(x)\) of a given function \(f \sim p(f)\) sampled from some prior.
Due to its in-context, foundational nature, our method does not require practitioners to train a particular surrogate model.
Meanwhile, due to its direct nature, our method does not require expensive and brittle acquisition function optimization. 
This results in substantial speed-ups, especially for parallel suggestions, as illustrated in Figure~\ref{fig:speed}. 
We refer to our method as \emph{\textbf{F}ully \textbf{I}n-Context \textbf{B}ayesian \textbf{O}ptimization}, or \emph{\textbf{FIBO}} for short.

FIBO simplifies the modeling and the inner optimization part of BO.
Nevertheless, although there is no explicit acquisition function maximization in FIBO, we prove that it is equivalent to Thompson sampling \citep{thompson1933sampling}.
Thus, FIBO is indeed a principled approach to BO.
Finally, we demonstrate the performance of FIBO in standard synthetic test functions and real-world benchmarks.
Despite being more than \( 10 \times \) faster than traditional BO approaches, it matches their performance consistently.

%% file: contents/02_background.tex
\section{Preliminaries}
\label{sec:bg}

Here, we discuss the necessary background to derive our method: (i) Bayesian optimization, (ii) in-context probabilistic surrogates, and (iii) deep generative models.

\subsection{Bayesian Optimization}

Let $f : \gX \rightarrow \sR$ denote an unknown objective function on a space \(\gX \subset \sR^d\).
Our goal is to find an optimal point $\vx^*  \in \argmax_{\vx \in \gX} f(\vx)$.
We are interested in the case where $f$ is expensive to compute and the input domain cannot be explored exhaustively.
We are limited to querying the objective function a few times and only have access to its output value \(f(\vx)\).

Bayesian optimization \citep[BO;][]{Mockus1975,garnett_bayesoptbook_2023} addresses this problem by using two components: a surrogate model and an acquisition function.
The surrogate consists of a probabilistic model $p(f|\train)$ that captures our posterior beliefs about the objective function \(f\) given our current dataset \(\D\).
Meanwhile, the acquisition function $\alpha : \gX \rightarrow \sR$ scores points in the input domain representing our preference over the next locations to be queried. At each step, the next point is selected via $\argmax_{\vx \in \gX} \alpha(\vx;\train)$.

One example of an acquisition function is Thompson sampling \citep[TS;][]{thompson1933sampling}, which is defined as $\alpha_\text{TS}(\vx;\train) = \smash{\hat{f}}(x)$, where $\smash{\hat{f}} \sim p(f|\train)$.
It is equivalent to sampling the next point from the posterior distribution over the optimal point, that is, $\vx_{t+1} \sim p(\vx_*|\train)$ \citep{shahriari_taking_2016}.
The inherent randomness of TS and the fact that it faithfully follows the posterior belief ensures good exploration-exploitation balance, making it widely used in practice \citep{kristiadi2024asyncBO}.

There are applications for which it is beneficial or even necessary to run multiple evaluations of the objective function in parallel.
In this high throughput setting, we are interested in simultaneously suggesting a collection \((\vx_1,...,\vx_q)\) of \(q\) points to be evaluated.
Evaluating the joint acquisition function over the entire batch poses complex computational and optimization problems.
Popular solutions include sequential simulation~\citep{Ginsbourger2010}, in which the points are optimized greedily through \(q\) steps of standard BO, and MC approaches that can be applied when the posterior distribution is Gaussian~\citep{Wilson2018Maximizing}.
Importantly, Thompson sampling allows trivial batch construction by simply sampling \(q\) points from the posterior~\citep{hernandez-lobato_parallel_2017}.

\subsection{In-Context Learning}

In-context learning refers to algorithms that learn from a few examples provided at test time, without updating any parameter value \citep{Xie2021AnEO}.
In a supervised learning setting it refers to estimating $p(f(\vx)| \{\vx_i,f(\vx_i)\}_{i=1}^k, \vx_\text{query})$---the probability of a function \(f\) on a query point \(\vx_\text{query}\) given \(k\) in-context examples \(\{\vx_i,f(\vx_i)\}_{i=1}^k\).
It can also be formulated for unsupervised tasks, such as generation, as sampling from $p(\vx_\text{query}|\vx_1,...,\vx_k)$ such that $\vx_\text{query}$ comes from the same distribution as the context examples.

Connecting to the familiar application of in-context learning in large language models, $\vx_k$ represents the $k$-th input example (e.g., a text prompt), and $f(\vx_k)$ is the corresponding output (e.g., a predicted continuation) in supervised tasks like translation or question-answering~\citep{brown2020gpt3}.

\subsection{Deep Generative Model}

Deep generative models aim to approximate a non-trivial and high-dimensional distribution \(p(\vx)\) by only having access to its samples. 
In this work, we are interested in conditional generative models of the form \(p_\theta(\vx|\vc)\) that make use of a neural network to generate samples conditioned on a context vector \(\vc \in \sR^c\).
This is the setting used in many popular applications such as text-to-image~\citep{Ramesh2021ZeroShotTG}, image-to-image~\citep{xiao2024omnigen} or even text-to-molecule generation~\citep{gong2024textguidedmoleculegenerationdiffusion}.
In addition to the context vector, the model usually receives as input a latent vector \(\vz\) that comes from a base distribution \(p_\vz(\vz)\) that is easy to sample from (e.g., multivariate Gaussian).

One family of such models are normalizing flows~\cite{Papamakarios2019NormalizingFF}, which transform a sample from the base distribution into a point from the complex data distribution through an invertible and differentiable transformation \(T_\theta(\vz;\vc)\) parametrized by a neural network. 
The model is trained to maximize the data likelihood using the change-of-variables formula: \(p(\vx) = p_\vz(T^{-1}_\theta(\vx;\vc))|\det \smash{J_{T^{-1}_\theta}}(\vx)|\), where \(\smash{J_{T^{-1}_\theta}}\) denotes the Jacobian matrix of the inverse transformation. 
Many efforts have been directed towards designing powerful transformations that can be efficiently inverted and have Jacobian determinants that are simple to compute.

%% file: contents/03_method.tex
\section{Fully In-Context Bayesian Optimization (FIBO)}
\label{sec:method}

Our goal is to develop a method that is able to perform Bayesian optimization completely in-context.
That is, we aim to directly suggest the input point of the next observation with no surrogate fitting or acquisition function maximization.
In order to achieve this goal, we propose FIBO, which amounts to a pretrained generative model that works across different objective functions during test time.

\subsection{In-Context Thompson Sampling}

Recall that Thompson sampling is characterized by sampling from a posterior distribution over a function \(f\)'s optimum point \(p(\vx_* \vert \D)\) after observing data \(\train\).
FIBO amounts to learning such a distribution through a model $p_\theta(\vx_* \vert \train)$ parametrized by \(\theta\), which is composed of two parts: an encoder and a generative head.
The encoder \(\gE_{\theta_\text{enc}}:2^{\gX \times \sR} \rightarrow \sR^c\) accepts a set of data points of variable cardinality that is summarized into a fixed-dimension context vector \(\vc = \gE_{\theta_\text{enc}}(\train)\). 
This context vector is passed to the generative head (a normalizing flow) \(\gG_{\theta_\text{gen}}: \sR^d \times \sR^c \rightarrow \gX\) alongside a random vector $\vz \in \sR^d$ to produce a sample from the posterior distribution over the optimal point \(\hat{\vx}_* = \gG_{\theta_\text{gen}}(\vz;\vc)\).
Combining both components we can train an end-to-end model parametrized by \(\theta = (\theta_\text{enc}, \theta_\text{gen})\) minimizing the negative log-likelihood loss function:
\begin{equation}
  \label{eq:loss}
  \gL(\theta) = -\E_{\vx_*,\train \sim p(f, \vx_*,\train)}[\log p_\theta(\vx_*|\train)] .
\end{equation}

Given a context dataset \(\train\), we can then use the pretrained \(p_\theta(\vx_* \vert \train)\) to directly sample the optimum point.
In \cref{subsec:analysis} we shall show that this sample comes from the posterior derived from some prior over functions after observing \(\D\).
But first, we clarify where this prior comes from and how to efficiently pretrain this model.

\subsection{Model Pretraining}
\label{sec:pretraining}

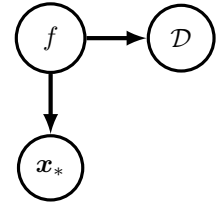
\begin{wrapfigure}[10]{r}{0.3\textwidth}
    \centering
    \vspace{-1.75em}
    \input{figs/dag.tikz}
    \caption{Factorization of the data-generating distribution.}
    \label{fig:dag}
\end{wrapfigure}

To generate the data for pretraining FIBO, we need to sample pairs \((\vx_*, \train)\)'s from the joint distribution \(p(f, \vx_*, \train)\).
Consider a factorization of such a distribution to the graphical model in \cref{fig:dag}:
\begin{equation}
  \label{eq:joint}
  p(f,\vx_*,\train) = p(\vx_*|f)p(\train|f)p(f)
\end{equation}
This factorization makes sense in BO since, given a function \(f\), the \emph{true} optimum \(\vx_*\) does not depend on the historical data \(\D\).
Thus, it allows us to efficiently sample \((\vx_*,\train) \) from \(p(f, \vx_*, \D)\) because \(\vx_* \perp\!\!\!\perp \train \mid f\), i.e., we can sample the function first and then separately sample the optimal point and the dataset. 
Thus, we can perform on-the-fly data augmentation by subsampling our dataset while keeping the optimal point of the corresponding function fixed.

Notice \(f\) is independent of \(\train\).
This implies that we only need a prior over functions \(p(f)\) for the data sampling process, as in \citet{muller_transformers_2023}. 
The final data-gathering process for pretraining $p_\theta(\vx_*|\train)$ is summarized in \cref{alg:data-collection} and full details in Appendix~\ref{apdx:training}.

\begin{figure}[t]
    \centering
    \includegraphics[width=\linewidth]{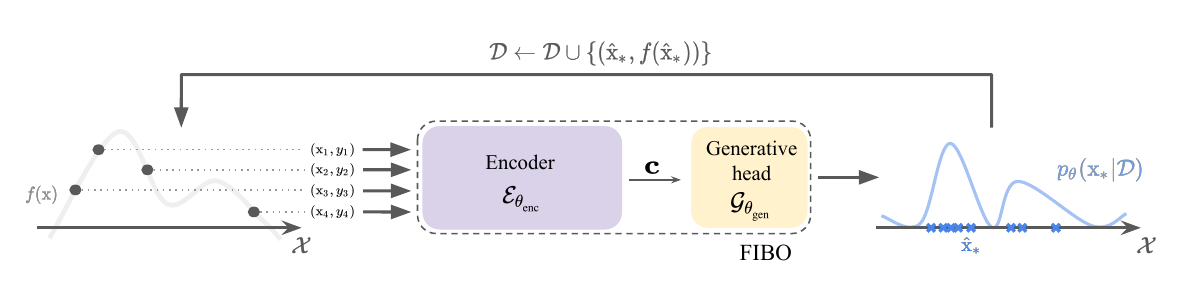}
    \vspace{-1.5em}    
    \caption{
        Illustration of a Bayesian optimization loop using FIBO. 
        The model receives a dataset as input and produces samples from the posterior distribution over the optimal point.
    }
    \label{fig:model}
\end{figure}

\subsection{Analysis}
\label{subsec:analysis}

Importantly, FIBO is an approximation of Thompson sampling under the prior distribution \(p(f)\) used for pretraining. 
The following theorem shows that the proposed method is a principled approach to BO and retains the properties of TS while being more efficient.

\vspace{0.5em}

\begin{theorem}
\label{thm:ts}
    An in-context generative model \(p_\theta(\vx_*|\train)\) trained by minimizing the loss $\gL(\theta)$ in \eqref{eq:loss} using samples $(\vx_*,\train)$ from $p(f, \vx_*,\train)$ defined in \eqref{eq:joint} is an approximation of Thompson sampling with $p(f|\train)$ defined by the chosen prior over function $p(f)$.
\end{theorem}
%
%
\begin{proof}[Proof Sketch]
    The result follows from rewriting the expected negative log-likelihood as an expectation over datasets of the cross entropy between the \(p_\theta(\vx_*|\train)\) and \(p(\vx_*|\train)\). Using the relationship between cross entropy and KL divergence we show that we are minimizing the latter between the two distributions. 
    Complete proof in Appendix~\ref{apdx:proof}.
\end{proof}

This result allows us to use Thompson sampling following the principles of in-context learning. In contrast, other surrogate models used in BO, such as GPs or Laplace-approximated BNNs use considerably more computation. 
The former requires fitting and approximate sampling using spectral techniques \cite{Rahimi2007RandomFF}, and the latter also requires fitting and computing the network's Jacobian.

\begin{figure}[t]
    \begin{minipage}{0.475\textwidth}
        \begin{algorithm}[H]
            \caption{Pretraining data generation }
            \label{alg:data-collection}
            \begin{algorithmic}
            \Require Prior \(p(f)\), minimum dataset size $N_\text{min}$, maximum dataset size $N_\text{max}$
            \Ensure Pretraining data \(\D_\text{PT}\)
            \Statex
            \State Initialize pretraining dataset $\train_\text{PT} = \emptyset$
            \Repeat
            \State Sample objective function $f \sim p(f)$
            \State $\vx_* = \mathtt{gradient\_ascent}(f)$ 
            \State $n \gets \gU(N_\text{min},N_\text{max})$
            \State $X = (\vx_1,...,\vx_n), \vx_i \sim \gU(\gX)$
            \State $\train = \{(\vx_i,f(\vx_i)) : \vx_i \in X\}$
            \State $\train_\text{PT} = \train_\text{PT} \cup \{(\vx_*,\train)\}$
            \Until{Desired amount of data is collected}
            \end{algorithmic}
        \end{algorithm}
    \end{minipage}
    \hfill
    \begin{minipage}{0.475\textwidth}
        \begin{algorithm}[H]
          \caption{BO loop with FIBO}
          \label{alg:bo-loop}
          \begin{algorithmic}
            \Require Model $q_\theta$, initial data $\train_0 = \{(\vx_1, y_1),...,(\vx_k, y_k)\}$, objective $f$, batch size $q$, number of iterations $T$

            \Statex 
            \For{$t \gets 0$ to $T-1$}
            \State $\vc \gets  \gE_{\theta_\text{enc}}(\train_t)$
            \For{$i \gets 1$ to $q$}
            \State $\vz_i \gets \gN(\mathbf{0},\mathbf{I})$
            \State $\vx_{i} \gets \gG_{\theta_\text{gen}}(\vz_i;\vc)$
            \State $y_{i} \gets f(\vx_{i})$
            \EndFor
            \State $\train_{t+1} = \train_t \cup \{(\vx_{i},y_{i})\}_{i=1}^{q}$
            \EndFor
          \end{algorithmic}
        \end{algorithm}
    \end{minipage}
\end{figure}

\subsection{Using FIBO for Bayesian Optimization}

Given a model pretrained as described above, we are now able to perform BO completely in context.
This is done by sampling from the generative model at each step, as described in \cref{alg:bo-loop}.

\begin{wrapfigure}[13]{r}{0.35\textwidth}
    \centering
    \vspace{-1.5em}
    \includegraphics[width=\linewidth]{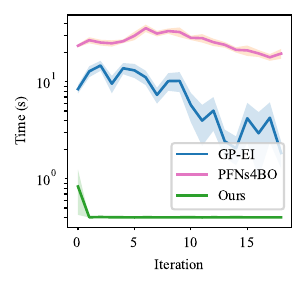}
    \vspace{-2em}
    \caption{Time comparison on \(4\)D Ackley function with \(q=10\)}
    \label{fig:time}
\end{wrapfigure}

Since the model is pretrained, no surrogate fitting is performed at test time.
In addition, because we directly sample from the posterior over the optimal point there is no explicit acquisition function maximization. 
Therefore, FIBO bypasses the iterative optimization steps present in both stages of Bayesian optimization---instead performing a single forward pass through the deep generative model.

As a consequence, FIBO is significantly faster than the alternatives. Figure \ref{fig:time} shows a comparison between FIBO and a Gaussian Process in terms of the time taken to generate a suggestion of a batch with 10 points on the 4D Ackley function. Even though the GP implementation makes use of modern GPU acceleration and efficient algorithms \citep{Balandat2019BoTorchPB}, FIBO is more than one order of magnitude faster.

%% file: figs/dag.tikz
\begin{tikzpicture}[scale=1, transform shape, node distance=0.8cm]
\tikzstyle{every node}=[draw=black, circle, very thick, text width=0.5cm, align=center]

\node[] (f) {$f$};
\node[below= of f] (x) {$\vx_*$};
\node[right= of f] (D) {$\train$};

\draw[-latex, ultra thick] (f.south) to (x.north);
\draw[-latex, ultra thick] (f.east) to (D.west);

\end{tikzpicture}

%% file: contents/04_related.tex
\section{Related Work}
\label{sec:related}

\textbf{In-context surrogate models} \quad
Previous works explored alternatives to avoid re-training the surrogate model at each step of the BO loop.
In general, solutions to this problem follow the Neural Processes~\cite{garnelo_conditional_2018} framework, modelling the predictive distribution for a test point given a set of labelled points.
PFNs4BO~\citep{muller_pfns4bo_2023} makes use of a Prior-Fitted Network \citep[PFN;][]{muller_transformers_2023} --- a pretrained Transformer that outputs continuous distributions using a binned representation  --- as surrogate. 
The model is pretraining by sampling functions from a predetermined prior distribution.
Notice that PFNs4BO only addresses the surrogate model. It still requires acquisition function maximization at every step of the loop.
Furthermore, because their predictive posterior takes the form of a binned distribution it cannot enjoy the benefits of batch suggestions implemented with efficient MC approximations.
Thus relying on slow sequential simulation techniques.

\textbf{Generative models for black-box optimization} \quad
Inverse approaches for black-box optimization learn the mapping from function values to inputs in the domain of the objective function. While early work was based on generative adversarial networks~\citep{kumarlevine2020model}, recent developments have successfully used the higher generation power of diffusion models.
\citet{krishnamoorthy_diffusion_2023} proposed a method for offline black-box optimization that trains a conditional diffusion model on a given dataset.
During evaluation, the model generates samples conditioned on the higher function value observed in the training dataset.
Diff-BBO~\citep{wu_diff-bbo_2024} applies an analogous technique in an online fashion, re-fitting an ensemble of diffusion models after collecting every new batch of data.
More recently, \citet{yun2025posteriorinferencediffusionmodels} proposed DiBO, which at each step trains an unconditional diffusion model and an ensemble of regression models that are used together to sample the next points using local search.
Importantly, these methods require extensive test-time compute in order to fit the multiple models at each step of the objective function optimization loop.
For this reason, this family of models are even more computationally expensive than traditional BO and in-context surrogate modelling, thus it is out of the scope of our work.

\textbf{Generative Models for Approximate Bayesian Inference} \quad
Bayesian statistics is interested in inferring the posterior distribution over the model parameters, which is often complex and high-dimensional. 
Therefore, researchers have investigated ways in which deep generative models can be used as tools for approximate Bayesian inference.
BayesFlow~\citep{bayesflow_2020_original} introduces a framework that makes use of DeepSets~\citep{Zaheer2017Deepsets} and normalizing flows to learn the posterior distribution using data-parameters pairs generated via simulation from a predetermined model. The trained model is then used to perform amortized Bayesian inference via sampling.
This line of work was extended by~\citep{mittal2025amortizedincontextbayesianposterior}, investigating the effect of different encoders, such as GRU~\citep{cho-etal-2014-learning} and Transformer~\citep{Vaswani2017AttentionIA}, and comparing objective functions based on the forward and backward KL divergences.
\citet{reuter2025transformerslearnbayesianinference} explored a more general problem in which there is no predetermined model, instead they proposed a model pretrained on a large enough dataset that is capable of approximating a large class of distributions. 
Their model, which consists of a modified Transformer encoder and continuous normalizing flows, shows state-of-the-art performance for generalized linear models and latent factor models.

%% file: contents/05_experiments.tex
\section{Experiments}
\label{sec:experiments}

\textbf{Architecture} \quad
For the encoder, we use the Transformer~\citep{Vaswani2017AttentionIA} model from the PFN pretrained on the BNN prior released by the authors in PFNs4BO~\citep{muller_pfns4bo_2023}.
The context vector is extracted by taking the average of the Transformer's output vector of each data point. 
We observe that starting from a pretrained model helps with performance even if the prior for which we are finetuning is not the same as the one used for pretraining.
The generative model used is an autoregressive neural spline flow~\citep{Durkan2019Spline} from the normflows library~\citep{Stimper2023}. Implementation details are provided in Appendix~\ref{apdx:implementation}.

\textbf{Pretraining Prior} \quad
We generate the data to pretrain our model using a GP prior as described in \citet{hernandez2014entropysearch}. 
Using the Fourier dual representation of the RBF kernel~\citep{Rahimi2007RandomFF} we are able to sample parametric approximations of the GP.
This allows us to sample \((\vx_*, \train)\) pairs efficiently using standard linear algebra packages.
More specifically, once we sample an instance of a \(f \sim \text{GP}\) we use L-BFGS-B~\cite{Byrd1995lbfgs} with multiple restarts to obtain \(\vx_*\) and sample the dataset \(\train\) uniformly.
In addition, we perform rejection sampling to ensure that the marginal distribution over the optimal point is uniform across the input domain.

\textbf{Baselines} \quad
We compare our method against a Gaussian Process with Matern-$\nicefrac{5}{2}$ and Log Expected Improvement~\citep{ament_unexpected_2024} (\textsc{GP-EI}), linearized Laplace approximation~\citep{kristiadi_promises_2023} (\textsc{LLA}), \textsc{PFNs4BO}~\citep{muller_pfns4bo_2023} with HEBO prior, and random search (\textsc{RS}). For GP-EI and LLA we use MC acquisition functions~\citep{wilson_maximizing_2018} to obtain suggestion batches. For PFNs4BO, we adapt their implementation to perform batch selection via sequential simulation using the expected value under the model~\citep{Ginsbourger2010}.
The LLA baseline is highly memory intensive in the batch setting, for this reason, we only report their results up to the batch size that fits in GPU memory.
Details are provided in Appendix~\ref{apdx:implementation}.

\textbf{Evaluation} \quad
For all experiments we use the \textsc{GAP} measure~\citep{Jiang2020Binoculars} to evaluate the performance: $\text{GAP}~=~(y_i - y_0)/(y_* - y_0)$. Whenever the objective function does not have a known optimal point we take it to be the best result across all runs and methods. We also compare the wall-clock time (in seconds) taken to propose the next points at each iteration. For all tasks we perform a total of 200 function evaluations, using batch sizes \(q \in \{10, 20, 50\}\) and setting the number of initial points to the corresponding batch size.

%% file: contents/06_results.tex
\subsection{Well-Specified Functions}

\begin{figure}[t]
    \centering
    \includegraphics[width=\linewidth]{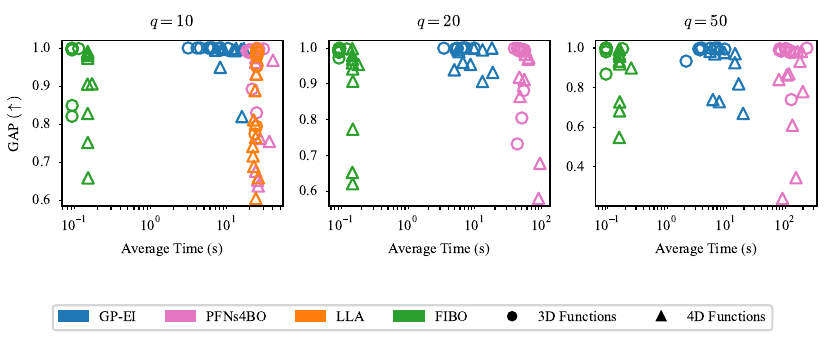}
    \caption{
    Comparison of Bayesian optimization methods on function sampled from FIBO's pretraining prior across different batch sizes \(q \in \{10, 20, 50\}\).
    Each subplot shows the final GAP (higher is better) versus the average wall-clock time (log scale) per run.
    }
    \label{fig:func-prior}
\end{figure}

We start by considering functions that come from the same distribution used to train FIBO. 
To this end, we sample 10 functions from the prior for each dimension. 
All functions are defined on the unit hypercube of their corresponding dimension. 

The results presented in Figure~\ref{fig:func-prior} show that FIBO is up to two and three orders of magnitude faster than GP and PFNs4BO, respectively, as the batch size increases. 
Importantly, the gains in speed do not interfere with the performance in optimization.
Across all runs FIBO shows similar performance in terms of GAP when compared to PFNs4BO, and gets closer to the GP as the batch size increases.
Importantly, given that the functions in this set of experiments come from a GP prior, it is expected that the GP would have the best performance, considering the surrogate is re-fit at each step.

\subsection{Synthetic Functions}

\begin{figure}[t]
    \centering
    \includegraphics[width=\linewidth]{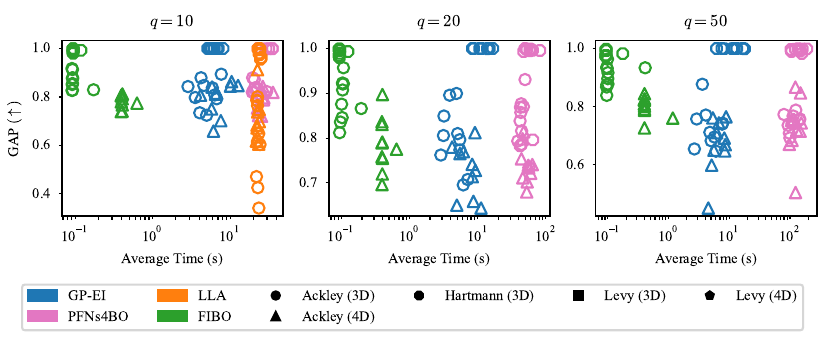}
    \caption{
    Comparison of Bayesian optimization methods on standard synthetic functions across different batch sizes \(q \in \{10, 20, 50\}\).
    Each subplot shows the final GAP (higher is better) versus the average wall-clock time (log scale) per run.
    }
    \label{fig:synthetic}
\end{figure}

We also evaluate the methods on a suite of synthetic functions commonly used in Bayesian Optimization: Ackley function on \([-32.768, 32.768]^{d} \subset \R^{d}\), Levy function on \([-10, 10]^{d} \subset \R^{d}\), and Rosenbrock function on \([-5, 10]^{d} \subset \R^{d}\), taking \(d \in \{3,4\}\) for all three functions.
We also use the Hartmann function on \([0, 1]^{3} \subset \R^{3}\).

The results are presented in Figure \ref{fig:synthetic}. Across all methods, there is higher variance in the final GAP, with harder tasks, such as the Ackley function showing lower scores.
Once more it is possible to see that FIBO is exceptionally faster than the baselines while preserving the optimization quality.
In fact, the aggregated results in Table~\ref{tab:results} show that on average FIBO is similar to both GP-EI and PFNs4BO in terms of GAP.

\subsection{Real-World Chemistry Tasks}

\begin{figure}[b]
    \centering
    \includegraphics[width=\linewidth]{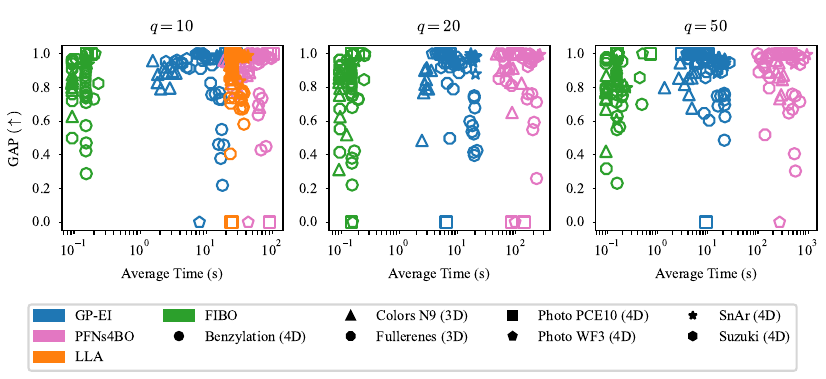}
    \caption{
    Comparison of Bayesian optimization methods on various chemistry tasks from Olympus toolkit across different batch sizes \(q \in \{10, 20, 50\}\).
    Each subplot shows the final GAP (higher is better) versus the average wall-clock time (log scale) per run.
    }
    \label{fig:olympus}
\end{figure}

We conduct experiments on a wide variety of chemistry benchmarks for which an oracle function is provided by the Olympus toolkit~\cite{hase_olympus_2021}. 
All of the used tasks have continuous input domains and vary in dimension. Complete information about the tasks is provided in Appendix~\ref{apdx:olympus}.

We present the results in Figure \ref{fig:olympus}. The overall results are on par with the previous experiments, showing FIBO faster than the other baselines with little to no drop in GAP performance.
In addition, this set of experiments shows the clear advantage of the in-context nature of FIBO.
While GP-EI and PFNs4B0 are slower than other tasks with the same batch size due to a harder inner optimization loop (i.e., maximization of the acquisition function), FIBO's execution time stays the same.
As presented in Table~\ref{tab:results}, this results in FIBO being up to \(100 \times\) faster than GP-EI and up \(1000 \times\) faster than PFNs4BO with no significant drop in GAP performance.

In addition, we perform hypothesis testing to assess if the difference between the methods is statistically significant.
We perform the tests in a pairwise fashion for each batch size \(q \in \{10, 20, 50\}\) considering GAP and wall-clock time separately and together.
For the unidimensional tests we use Wilcoxon rank-sum test and for the bi-dimensional analysis we use permutation test of the groups means.
As the p values presented in Table~\ref{tab:pvalues} there is statistical significancy between FIBO and the baselines when considering wall-clock time alone and both wall-clock time and final GAP together.
Importantly, the statistical analysis shows no significant difference between the final GAP obtained by FIBO and all other baselines.

\begin{table}[t]
    \centering
    \caption{
    Aggregated results of the Bayesian optimization methods across all runs and tasks for all task sets and batch sizes \(q \in \{10, 20, 50\}\).
    The best final GAP (higher is better) and wall-clock time (lower is better) per group is highlighted taking into consideration the standard error.
    }
    \resizebox{\textwidth}{!}{\begin{tabular}{llcccccc}\toprule
        \multirow{2}{*}{Task Set} & \multirow{2}{*}{Method} & \multicolumn{2}{c}{\(q=10\)} & \multicolumn{2}{c}{\(q=20\)} & \multicolumn{2}{c}{\(q=50\)} \\
        \cmidrule(lr){3-4}\cmidrule(lr){5-6}\cmidrule(lr){7-8}
         & & GAP & Time & GAP & Time & GAP & Time \\
        \midrule
        \multirow{4}{*}{Prior} & GP-EI & \(0.99 \pm 0.01\) & \(9.06 \pm 0.88\) & \(0.98 \pm 0.01\) & \(8.28 \pm 0.93\) & \(0.93 \pm 0.02\) & \(7.88 \pm 1.02\)\\ 
        & PFNs4BO & \(0.90 \pm 0.03\) & \(25.05 \pm 1.17\) & \(0.91 \pm 0.03\) & \(55.42 \pm 3.07\) & \(0.85 \pm 0.05\) & \(129.04 \pm 9.45\)\\ 
        & LLA & \(0.85 \pm 0.03\) & \(23.99 \pm 0.24\) & --- & --- & --- & ---\\ 
        & FIBO & \(0.93 \pm 0.02\) & \(0.12 \pm 0.01\) & \(0.93 \pm 0.02\) & \(0.12 \pm 0.01\) & \(0.92 \pm 0.03\) & \(0.13 \pm 0.01\)\\
        \midrule
        \multirow{4}{*}{Synthetic} & GP-EI & \(0.92 \pm 0.01\) & \(8.69 \pm 0.48\) & \(0.90 \pm 0.02\) & \(8.39 \pm 0.44\) & \(0.87 \pm 0.02\) & \(7.76 \pm 0.51\)\\ 
        & PFNs4BO & \(0.92 \pm 0.01\) & \(25.35 \pm 0.61\) & \(0.90 \pm 0.02\) & \(51.81 \pm 1.30\) & \(0.89 \pm 0.02\) & \(126.07 \pm 3.26\)\\ 
        & LLA & \(0.82 \pm 0.03\) & \(23.30 \pm 0.11\) & --- & --- & --- & ---\\ 
        & FIBO & \(0.92 \pm 0.01\) & \(0.17 \pm 0.02\) & \(0.93 \pm 0.01\) & \(0.18 \pm 0.02\) & \(0.93 \pm 0.01\) & \(0.20 \pm 0.03\)\\
        \midrule
        \multirow{4}{*}{Chemistry} & GP-EI & \(0.89 \pm 0.02\) & \(11.58 \pm 0.86\) & \(0.87 \pm 0.02\) & \(9.85 \pm 0.73\) & \(0.88 \pm 0.02\) & \(11.14 \pm 0.81\)\\ 
        & PFNs4BO & \(0.88 \pm 0.03\) & \(53.66 \pm 2.72\) & \(0.87 \pm 0.03\) & \(125.15 \pm 7.54\) & \(0.89 \pm 0.02\) & \(348.05 \pm 23.48\)\\ 
        & LLA & \(0.88 \pm 0.02\) & \(27.76 \pm 0.63\) & --- & --- & --- & ---\\ 
        & FIBO & \(0.87 \pm 0.02\) & \(0.14 \pm 0.00\) & \(0.82 \pm 0.03\) & \(0.14 \pm 0.00\) & \(0.85 \pm 0.02\) & \(0.17 \pm 0.01\)\\
    \bottomrule
    \end{tabular}}

    \label{tab:results}
\end{table}

\begin{wraptable}[18]{r}{0.52\textwidth}

\centering
\vspace{-1.25em}

\caption{P-values for pairwise hypothesis tests between FIBO and baselines across all tasks and runs on the real-world chemistry benchmark.}
\label{tab:pvalues}
\vspace{0.5em}

\footnotesize

\begin{tabular}{llccc}
\toprule
& & GP-EI & PFNs4BO & LLA \\
\midrule
\multirow{3}{*}{\(q=10\)}& Both & \(0.00\) & \(0.00\) & \(0.00\) \\
& Time & \(0.00\) & \(0.00\) & \(0.00\) \\
& GAP & \(0.15\) & \(0.12\) & \(0.32\) \\
\midrule
\multirow{3}{*}{\(q=20\)}& Both & \(0.00\) & \(0.00\) & -- \\
& Time & \(0.00\) & \(0.00\) & -- \\
& GAP & \(0.10\) & \(0.09\) & -- \\
\midrule
\multirow{3}{*}{\(q=50\)}& Both & \(0.00\) & \(0.00\) & -- \\
& Time & \(0.00\) & \(0.00\) & -- \\
& GAP & \(0.26\) & \(0.05\) & -- \\
\bottomrule
\end{tabular}

\end{wraptable}

%% file: contents/08_conclusion.tex
\section{Conclusion}
\label{sec:conclusion}

We introduced FIBO, a pretrained model that performs completely in-context Bayesian optimization with no surrogate fitting or acquisition function maximization.
We prove that by pretraining FIBO on functions drawn from a chosen prior we are approximating Thompson sampling, demonstrating that the method is a principled approach to Bayesian optimization.
Through experiments in both synthetic and real-world benchmarks, we show that FIBO's results are on par with other existing methods in the literature while being significantly faster, especially in the batch optimization setting.
In future work, we plan on scaling FIBO both in terms of pretraining data and model size, following a series of recent works that demonstrate the powers of scaling in many different domains.

\textbf{Limitations} \quad 
We propose a completely in-context method for BO that enjoys multiple orders of magnitude speed up compared to traditional BO methods. 
However, we focus on problems in lower dimensions since they are ubiquitous throughout science and engineering.
Moreover, our method does not solve the problem of the inherently long experiment time for computing \(f(\vx)\).
Nevertheless, FIBO is useful in batched, high-throughput, or multi-fidelity scenarios where the computation of \(f(\vx)\) is less of a problem.

%% file: contents/09_app_data_gen.tex
\section{Data Generation Details}
\label{apdx:training}

As described in Section~\ref{sec:pretraining} we are interested in sampling pairs of the form \((\vx_*,\train)\) where \(\vx_* = \arg \max_{\vx \in \gX} f(x)\) and \(\train=\{(\vx,f(\vx)\}_{i=1}^N\) for some function \(f:\gX\rightarrow\R\) sampled from a prior distribution \(p(f)\).
Considering the factorization presented in the graphical model in Figure~\ref{fig:dag} it is clear that we can sample \(f,\vx_*,\train\) triples using ancestral sampling. We first sample a function from the chosen prior, then sampling from \(p(\vx_*|f)\) is performed by finding the optimal point of \(f\), and finally to sample from \(p(\train|f)\) we obtain samples of input points uniformly over the domain and pass them through the function.
At the end, we can simply discard the function to obtain \((\vx_*,\train)\) pairs.

Importantly, having \(\vx_* \perp\!\!\!\perp \train \mid f\) allows us to perform extensive data augmentation.
As the dataset is independent of the optimum for a given function, we can have vary dataset without having the need to recompute \(\vx_*\). 
We leverage this fact by collecting a large dataset when generating the data and dynamically subsampling from it during pretraining.
This has the effect of exponentially augmenting the amount of pretraining data while constructing datasets of variable sizes.

In practice, we want to have a prior over functions that are non-convex, as that is the underlying structure of many applications studied in Bayesian optimization.
Therefore, obtaining \(\vx_* = \arg \max_{\vx \in \gX} f(x)\) for a known sampled function \(f\) is itself a complex task.
To deal with this hard global optimization problem we make use of common practices in literature, such as optimization methods that use second order information and performing multiple restarts.
Although this can entail in higher computational cost, the data generation is performed offline a single time, and can be highly parallelizible across multiple independent processes.

%% file: contents/10_app_proof.tex
\section{Proof of Theorem~\ref{thm:ts}}
\label{apdx:proof}

Let us begin by applying the definition of expected value to the loss function presented in \eqref{eq:loss}:

\begin{equation}
\begin{split}
\gL(\theta) &= -\E_{\vx_*,\train \sim p(\vx_*,\train)}[\log p_\theta(\vx_*|\train)] \\
&= -\int_\train \int_{\vx_*} p(\vx_*,\train) \log p_\theta(\vx_*|\train)
\end{split}
\end{equation}

Now we factor \(p(\vx_*,\train)\) using the product rule and rearrange the integral accordingly:

\begin{equation}
\label{eq:step2}
\gL(\theta) = -\int_\train p(\train) \int_{\vx_*} p(\vx_*|\train) \log p_\theta(\vx_*|\train)
\end{equation}

The outer integral in \eqref{eq:step2} can be seen an expectation over \(\train\), while the inner integral (with the negative sign in the front) is the definition of the cross-entropy between \(p(\vx_*|\train)\) and \(p_\theta(\vx_*|\train)\). Thus, we can re-write the loss as

\begin{equation}
\gL(\theta) = \E_{\train \sim p(\train)}[H(p(\vx_*|\train), p_\theta(\vx_*|\train))]
\end{equation}

Making use of the identity that relates cross-entropy and KL divergence we get

\begin{equation}
\gL(\theta) = \E_{\train \sim p(\train)}[KL(p(\vx_*|\train) || p_\theta(\vx_*|\train))] + \E_{\train \sim p(\train)}[H(p(\vx_*|\train))]
\end{equation}

Dropping the second term, as it does not depend on \(\theta\), we arrive at the final result

\begin{equation}
\gL(\theta) = \E_{\train \sim p(\train)}[KL(p(\vx_*|\train) || p_\theta(\vx_*|\train))]
\end{equation}

demonstrating that by minimizing \(\gL(\theta)\) in \eqref{eq:loss} we are minimizing the expected KL divergence between the true Thompson sampling distribution and the one modelled by the deep generative model.

%% file: contents/11_app_implementation.tex
\section{Implementation Details}
\label{apdx:implementation}

\paragraph{GP prior}
We sample \((\vx_*, \train)\) from an approximation of a GP prior as in~\citet{hernandez2014entropysearch}. A squared-exponential kernel \(k(\vx,\vx')=\gamma^2\exp(-0.5\sum_i(x_i-x'_i)/l_i^2)\) can be approximated using random Fourier features~\citep{Rahimi2007RandomFF} with the feature map \(\phi(\vx) = \sqrt{2\gamma^2/m}\cos(\mW\vx+\vb)\) where \(\mW\) and \(\vb\) are composed of \(m\) stacked samples from \(\gN(\vzero,\text{diag}(\vl^{-2}))\) and~\(\gU(0, 2\pi)\), respectively.
This enables the GP prior to be approximated by a linear model \(f(\vx) = \phi(\vx)\top\vbeta\), with \(\vbeta \sim \gN(\vzero,\mI)\).
Making use of this representation allows us to obtain \(\vx_* = \arg\max_{\vx \in \gX} f(\vx)\) and to sample \(\train\) trivially.

For our experiments we use hyperpriors \(l_i \sim \gU(0.01, 5)\) and \(\gamma^2 \sim \gU(1, 2)\).
In addition, we perform rejection sampling over the \((\vx_*, \train)\) pairs to ensure that the marginal distribution over the optima is uniform in the domain.
This last step is necessary to unsure that FIBO is not biased towards specific regions of the input space.

\paragraph{FIBO}
We pretrain one model per dimension using a dataset of 70,000 pairs sampled from the GP prior described above. The models are trained for 200 epochs, with batch size of 128, and using Adam optimizer with starting learning rate of \(10^{-4}\) and cosine scheduler. Training is performed in PyTorch on a single NVIDIA Quadro RTX 6000.

The architecture is composed of a Transformer encoder from PFNs4BO paper pretrained on the BNN prior\footnote{\url{https://github.com/automl/PFNs4BO}} and a neural spline flow generative head from the normflows library~\footnote{\url{https://github.com/VincentStimper/normalizing-flows}}. 
The output of the encoder is projected to \(\R^c\) with \(c=256\) and \(c=512\) for \(d=3\) and \(d=4\), respectively. Each flow block is comprised of 6 layers with 256 hidden units. For \(d=3\) we make use of 6 flow blocks, while for \(d=4\) we use 8 blocks.

\paragraph{Baselines}
For PFNs4BO we use the official implementation and weights provided by the authors, extending it to deal with batch suggestions using sequential simulation~\citep{Ginsbourger2010} and to deal with datasets of different sizes. All acquisition function optimization parameters are kept as they appear in the original paper~\citep{muller_pfns4bo_2023}.
We use the BoTorch~\citep{Balandat2019BoTorchPB} implementation of the GP model, optimizing the LogEI\citep{ament_unexpected_2024} with 10 restarts and 512 raw samples. 
Finally, for LLA we use their official BoTorch implementation\footnote{\url{https://github.com/wiseodd/laplace-bayesopt}} alongside the same acquisition function setup used for the GP.
FIBO, PFNs4BO, and GP-EI were executed on a NVIDIA Quadro RTX 6000, while LLA was executed on a NVIDIA A40 due to its higher memory requirements.

%% file: contents/12_app_olympus.tex
\section{Olympus Details}
\label{apdx:olympus}

The Olympus toolkit~\citep{hase_olympus_2021} provides a wide variety of chemistry benchmarks for black-box optimization.
It contains machine learning based emulators that are used as oracles for the different tasks.
This allows fast objective function evaluation while still providing evaluation on real-world applications. The detailed information each tasks used in our experiments is provided in Table\ref{tab:olympus}. 
For all tasks we use the \texttt{BayesNeuralNet} emulator.

\begin{table}
    \centering
    \caption{Descriptions of the tasks from Olympus toolkit we use in our experiments.}
    \resizebox{\textwidth}{!}{\begin{tabular}{lcl}
    \toprule
         Name & Input Dimension & Description   \\
        \midrule
        Colors N9 & 3D & Minimize distance to target color mixing red, green and blue dye  \\
        Fullerness & 3D & Maximize the mole fraction of o-xylenyl adducts of Buckminsterfullerenes   \\
        Photo PCE10 & 4D & Minimize photo-degradation of organic solar cells with PCE10 \\
        Photo WF3 & 4D & Minimize photo-degradation of organic solar cells with WF3  \\
        Benzylation & 4D & Minimize yield of inpurity in N-benzylation reaction \\
        SnAr & 4D & Mininize e-factor for a nucleophilic aromatic substitution following the SnAr mechanism\\
        Suzuki & 4D & Maximize Suzuki coupling yield by adjusting reaction parameters   \\
    \bottomrule
    \end{tabular}}
    \label{tab:olympus}
\end{table}

%% file: main.bbl
\begin{thebibliography}{50}
\providecommand{\natexlab}[1]{#1}
\providecommand{\url}[1]{\texttt{#1}}
\expandafter\ifx\csname urlstyle\endcsname\relax
  \providecommand{\doi}[1]{doi: #1}\else
  \providecommand{\doi}{doi: \begingroup \urlstyle{rm}\Url}\fi

\bibitem[Ament et~al.(2023{\natexlab{a}})Ament, Daulton, Eriksson, Balandat, and Bakshy]{ament_unexpected_2024}
Sebastian Ament, Samuel Daulton, David Eriksson, Maximilian Balandat, and Eytan Bakshy.
\newblock Unexpected improvements to expected improvement for {B}ayesian optimization.
\newblock In \emph{NeurIPS}, 2023{\natexlab{a}}.

\bibitem[Ament et~al.(2023{\natexlab{b}})Ament, Daulton, Eriksson, Balandat, and Bakshy]{wilson_maximizing_2018}
Sebastian Ament, Samuel Daulton, David Eriksson, Maximilian Balandat, and Eytan Bakshy.
\newblock Unexpected improvements to expected improvement for {B}ayesian optimization.
\newblock In \emph{NeurIPS}, 2023{\natexlab{b}}.

\bibitem[Auer(2003)]{Auer2003UsingCB}
Peter Auer.
\newblock Using confidence bounds for exploitation-exploration trade-offs.
\newblock \emph{JMLR}, 3\penalty0 (null), 2003.

\bibitem[Balandat et~al.(2020)Balandat, Karrer, Jiang, Daulton, Letham, Wilson, and Bakshy]{Balandat2019BoTorchPB}
Maximilian Balandat, Brian Karrer, Daniel Jiang, Samuel Daulton, Ben Letham, Andrew~G Wilson, and Eytan Bakshy.
\newblock Botorch: A framework for efficient monte-carlo bayesian optimization.
\newblock In \emph{NeurIPS}, 2020.

\bibitem[Brown et~al.(2020)Brown, Mann, Ryder, Subbiah, Kaplan, Dhariwal, Neelakantan, Shyam, Sastry, Askell, Agarwal, Herbert-Voss, Krueger, Henighan, Child, Ramesh, Ziegler, Wu, Winter, Hesse, Chen, Sigler, Litwin, Gray, Chess, Clark, Berner, McCandlish, Radford, Sutskever, and Amodei]{brown2020gpt3}
Tom Brown, Benjamin Mann, Nick Ryder, Melanie Subbiah, Jared~D Kaplan, Prafulla Dhariwal, Arvind Neelakantan, Pranav Shyam, Girish Sastry, Amanda Askell, Sandhini Agarwal, Ariel Herbert-Voss, Gretchen Krueger, Tom Henighan, Rewon Child, Aditya Ramesh, Daniel Ziegler, Jeffrey Wu, Clemens Winter, Chris Hesse, Mark Chen, Eric Sigler, Mateusz Litwin, Scott Gray, Benjamin Chess, Jack Clark, Christopher Berner, Sam McCandlish, Alec Radford, Ilya Sutskever, and Dario Amodei.
\newblock Language models are few-shot learners.
\newblock In \emph{NeurIPS}, 2020.

\bibitem[Byrd et~al.(1995)Byrd, Lu, Nocedal, and Zhu]{Byrd1995lbfgs}
Richard~H. Byrd, Peihuang Lu, Jorge Nocedal, and Ciyou Zhu.
\newblock A limited memory algorithm for bound constrained optimization.
\newblock \emph{SIAM Journal on Scientific Computing}, 16\penalty0 (5), 1995.

\bibitem[Cho et~al.(2014)Cho, van Merri{\"e}nboer, Gulcehre, Bahdanau, Bougares, Schwenk, and Bengio]{cho-etal-2014-learning}
Kyunghyun Cho, Bart van Merri{\"e}nboer, Caglar Gulcehre, Dzmitry Bahdanau, Fethi Bougares, Holger Schwenk, and Yoshua Bengio.
\newblock Learning phrase representations using {RNN} encoder{--}decoder for statistical machine translation.
\newblock In \emph{EMNLP}, 2014.

\bibitem[Durkan et~al.(2019)Durkan, Bekasov, Murray, and Papamakarios]{Durkan2019Spline}
Conor Durkan, Artur Bekasov, Iain Murray, and George Papamakarios.
\newblock Neural spline flows.
\newblock In \emph{NeurIPS}, 2019.

\bibitem[Feurer et~al.(2022)Feurer, Eggensperger, Falkner, Lindauer, and Hutter]{feurer2022auto}
Matthias Feurer, Katharina Eggensperger, Stefan Falkner, Marius Lindauer, and Frank Hutter.
\newblock Auto-sklearn 2.0: Hands-free automl via meta-learning.
\newblock \emph{JMLR}, 23\penalty0 (261), 2022.

\bibitem[Garnelo et~al.(2018)Garnelo, Rosenbaum, Maddison, Ramalho, Saxton, Shanahan, Teh, Rezende, and Eslami]{garnelo_conditional_2018}
Marta Garnelo, Dan Rosenbaum, Christopher Maddison, Tiago Ramalho, David Saxton, Murray Shanahan, Yee~Whye Teh, Danilo Rezende, and S.~M.~Ali Eslami.
\newblock Conditional neural processes.
\newblock In \emph{ICML}, 2018.

\bibitem[Garnett(2023)]{garnett_bayesoptbook_2023}
Roman Garnett.
\newblock \emph{{Bayesian Optimization}}.
\newblock Cambridge University Press, 2023.

\bibitem[Ginsbourger et~al.(2010)Ginsbourger, Le~Riche, and Carraro]{Ginsbourger2010}
David Ginsbourger, Rodolphe Le~Riche, and Laurent Carraro.
\newblock \emph{Kriging Is Well-Suited to Parallelize Optimization}.
\newblock Springer, 2010.

\bibitem[Gong et~al.(2024)Gong, Liu, Wu, and Wang]{gong2024textguidedmoleculegenerationdiffusion}
Haisong Gong, Qiang Liu, Shu Wu, and Liang Wang.
\newblock Text-guided molecule generation with diffusion language model.
\newblock \emph{Proceedings of the AAAI Conference on Artificial Intelligence}, 38\penalty0 (1), 2024.

\bibitem[Greenaway et~al.(2023)Greenaway, Jelfs, Spivey, and Yaliraki]{greenaway2023alchemist}
Rebecca~L Greenaway, Kim~E Jelfs, Alan~C Spivey, and Sophia~N Yaliraki.
\newblock From alchemist to ai chemist.
\newblock \emph{Nature Reviews Chemistry}, 7\penalty0 (8), 2023.

\bibitem[Griffiths and Hernández-Lobato(2020)]{griffiths2020constrained}
Ryan-Rhys Griffiths and José~Miguel Hernández-Lobato.
\newblock {C}onstrained {B}ayesian optimization for automatic chemical design using variational autoencoders.
\newblock \emph{Chem. Sci.}, 11, 2020.

\bibitem[H{\"a}se et~al.(2021)H{\"a}se, Aldeghi, Hickman, Roch, Christensen, Liles, Hein, and Aspuru-Guzik]{hase_olympus_2021}
Florian H{\"a}se, Matteo Aldeghi, Riley~J. Hickman, Lo{\"\i}c~M. Roch, Melodie Christensen, Elena Liles, Jason~E. Hein, and Al{\'a}n Aspuru-Guzik.
\newblock Olympus: a benchmarking framework for noisy optimization and experiment planning.
\newblock \emph{Machine Learning: Science and Technology}, 2\penalty0 (3), 2021.

\bibitem[Henr\'{a}ndez-Lobato et~al.(2014)Henr\'{a}ndez-Lobato, Hoffman, and Ghahramani]{hernandez2014entropysearch}
Jos\'{e}~Miguel Henr\'{a}ndez-Lobato, Matthew~W. Hoffman, and Zoubin Ghahramani.
\newblock Predictive entropy search for efficient global optimization of black-box functions.
\newblock In \emph{NeurIPS}, 2014.

\bibitem[Hern{\'a}ndez-Lobato et~al.(2017)Hern{\'a}ndez-Lobato, Requeima, Pyzer-Knapp, and Aspuru-Guzik]{hernandez-lobato_parallel_2017}
Jos{\'e}~Miguel Hern{\'a}ndez-Lobato, James Requeima, Edward~O. Pyzer-Knapp, and Al{\'a}n Aspuru-Guzik.
\newblock Parallel and distributed {T}hompson sampling for large-scale accelerated exploration of chemical space.
\newblock In \emph{ICML}, 2017.

\bibitem[Jiang et~al.(2020)Jiang, Chai, Gonzalez, and Garnett]{Jiang2020Binoculars}
Shali Jiang, Henry Chai, Javier Gonzalez, and Roman Garnett.
\newblock {BINOCULARS} for efficient, nonmyopic sequential experimental design.
\newblock In \emph{ICML}, 2020.

\bibitem[Jones et~al.(1998)Jones, Schonlau, and Welch]{Jones1998EfficientGO}
Donald~R. Jones, Matthias Schonlau, and William~J. Welch.
\newblock Efficient global optimization of expensive black-box functions.
\newblock \emph{Journal of Global Optimization}, 13, 1998.

\bibitem[Krishnamoorthy et~al.(2023)Krishnamoorthy, Mashkaria, and Grover]{krishnamoorthy_diffusion_2023}
Siddarth Krishnamoorthy, Satvik~Mehul Mashkaria, and Aditya Grover.
\newblock Diffusion models for black-box optimization.
\newblock In \emph{ICML}, 2023.

\bibitem[Kristiadi et~al.(2023)Kristiadi, Immer, Eschenhagen, and Fortuin]{kristiadi_promises_2023}
Agustinus Kristiadi, Alexander Immer, Runa Eschenhagen, and Vincent Fortuin.
\newblock Promises and pitfalls of the linearized {L}aplace in {B}ayesian optimization.
\newblock In \emph{Fifth Symposium on Advances in Approximate Bayesian Inference}, 2023.

\bibitem[Kristiadi et~al.(2024{\natexlab{a}})Kristiadi, Strieth-Kalthoff, Skreta, Poupart, Aspuru-Guzik, and Pleiss]{kristiadi_sober_2024}
Agustinus Kristiadi, Felix Strieth-Kalthoff, Marta Skreta, Pascal Poupart, Al\'{a}n Aspuru-Guzik, and Geoff Pleiss.
\newblock A sober look at {LLMs} for material discovery: {A}re they actually good for {B}ayesian optimization over molecules?
\newblock In \emph{ICML}, 2024{\natexlab{a}}.

\bibitem[Kristiadi et~al.(2024{\natexlab{b}})Kristiadi, Strieth-Kalthoff, Subramanian, Fortuin, Poupart, and Pleiss]{kristiadi2024asyncBO}
Agustinus Kristiadi, Felix Strieth-Kalthoff, Sriram~Ganapathi Subramanian, Vincent Fortuin, Pascal Poupart, and Geoff Pleiss.
\newblock How useful is intermittent, asynchronous expert feedback for {B}ayesian optimization?
\newblock In \emph{Sixth Symposium on Advances in Approximate Bayesian Inference-Non Archival Track}, 2024{\natexlab{b}}.

\bibitem[Kumar and Levine(2020)]{kumarlevine2020model}
Aviral Kumar and Sergey Levine.
\newblock Model inversion networks for model-based optimization.
\newblock In \emph{NeurIPS}, 2020.

\bibitem[Li et~al.(2024)Li, Rudner, and Wilson]{li_study_2023}
Yucen~Lily Li, Tim G.~J. Rudner, and Andrew~Gordon Wilson.
\newblock A study of {B}ayesian neural network surrogates for {B}ayesian optimization.
\newblock In \emph{ICLR}, 2024.

\bibitem[Mittal et~al.(2025)Mittal, Bracher, Lajoie, Jaini, and Brubaker]{mittal2025amortizedincontextbayesianposterior}
Sarthak Mittal, Niels~Leif Bracher, Guillaume Lajoie, Priyank Jaini, and Marcus Brubaker.
\newblock Amortized in-context {B}ayesian posterior estimation, 2025.

\bibitem[Mo{\v{c}}kus(1975)]{Mockus1975}
J.~Mo{\v{c}}kus.
\newblock On {B}ayesian methods for seeking the extremum.
\newblock In \emph{Optimization Techniques IFIP Technical Conference Novosibirsk, July 1--7, 1974}, 1975.

\bibitem[M{\"u}ller et~al.(2022)M{\"u}ller, Hollmann, Arango, Grabocka, and Hutter]{muller_transformers_2023}
Samuel M{\"u}ller, Noah Hollmann, Sebastian~Pineda Arango, Josif Grabocka, and Frank Hutter.
\newblock Transformers can do {B}ayesian inference.
\newblock In \emph{ICLR}, 2022.

\bibitem[M\"{u}ller et~al.(2023)M\"{u}ller, Feurer, Hollmann, and Hutter]{muller_pfns4bo_2023}
Samuel M\"{u}ller, Matthias Feurer, Noah Hollmann, and Frank Hutter.
\newblock {PFN}s4{BO}: {I}n-context learning for {B}ayesian optimization.
\newblock In \emph{ICML}, 2023.

\bibitem[Papamakarios et~al.(2021)Papamakarios, Nalisnick, Rezende, Mohamed, and Lakshminarayanan]{Papamakarios2019NormalizingFF}
George Papamakarios, Eric Nalisnick, Danilo~Jimenez Rezende, Shakir Mohamed, and Balaji Lakshminarayanan.
\newblock Normalizing flows for probabilistic modeling and inference.
\newblock \emph{Journal of Machine Learning Research}, 22\penalty0 (57), 2021.

\bibitem[Radev et~al.(2020)Radev, Mertens, Voss, Ardizzone, and K{\"o}the]{bayesflow_2020_original}
Stefan~T. Radev, Ulf~K. Mertens, Andreas Voss, Lynton Ardizzone, and Ullrich K{\"o}the.
\newblock {BayesFlow}: Learning complex stochastic models with invertible neural networks.
\newblock \emph{IEEE transactions on neural networks and learning systems}, 33\penalty0 (4):\penalty0 1452--1466, 2020.

\bibitem[Rahimi and Recht(2007)]{Rahimi2007RandomFF}
Ali Rahimi and Benjamin Recht.
\newblock Random features for large-scale kernel machines.
\newblock In \emph{NeurIPS}, 2007.

\bibitem[Ramesh et~al.(2021)Ramesh, Pavlov, Goh, Gray, Voss, Radford, Chen, and Sutskever]{Ramesh2021ZeroShotTG}
Aditya Ramesh, Mikhail Pavlov, Gabriel Goh, Scott Gray, Chelsea Voss, Alec Radford, Mark Chen, and Ilya Sutskever.
\newblock Zero-shot text-to-image generation.
\newblock In \emph{ICML}, 2021.

\bibitem[Rasmussen and Williams(2006)]{Rasmussen_Williams_gp_2006}
Carl~Edward Rasmussen and Christopher K.~I. Williams.
\newblock \emph{{Gaussian Processes for Machine Learning}}.
\newblock The MIT Press, 2006.

\bibitem[Reuter et~al.(2025)Reuter, Rudner, Fortuin, and Rügamer]{reuter2025transformerslearnbayesianinference}
Arik Reuter, Tim G.~J. Rudner, Vincent Fortuin, and David Rügamer.
\newblock Can transformers learn full {B}ayesian inference in context?
\newblock \emph{arXiv}, 2025.

\bibitem[Romero et~al.(2013)Romero, Krause, and Arnold]{romero2013navigating}
Philip~A Romero, Andreas Krause, and Frances~H Arnold.
\newblock Navigating the protein fitness landscape with {G}aussian processes.
\newblock \emph{Proceedings of the National Academy of Sciences}, 110\penalty0 (3), 2013.

\bibitem[Ruberg et~al.(2023)Ruberg, Beckers, Hemmings, Honig, Irony, LaVange, Lieberman, Mayne, and Moscicki]{ruberg2023bayesdrug}
Stephen~J Ruberg, Francois Beckers, Rob Hemmings, Peter Honig, Telba Irony, Lisa LaVange, Grazyna Lieberman, James Mayne, and Richard Moscicki.
\newblock Application of {B}ayesian approaches in drug development: starting a virtuous cycle.
\newblock \emph{Nature Reviews Drug Discovery}, 22\penalty0 (3), 2023.

\bibitem[Shahriari et~al.(2016)Shahriari, Swersky, Wang, Adams, and de~Freitas]{shahriari_taking_2016}
Bobak Shahriari, Kevin Swersky, Ziyu Wang, Ryan~P. Adams, and Nando de~Freitas.
\newblock Taking the human out of the loop: {A} review of {B}ayesian optimization.
\newblock \emph{Proceedings of the IEEE}, 2016.

\bibitem[Snoek et~al.(2012)Snoek, Larochelle, and Adams]{snoek2012boml}
Jasper Snoek, Hugo Larochelle, and Ryan~P Adams.
\newblock Practical {B}ayesian optimization of machine learning algorithms.
\newblock In \emph{NeurIPS}, 2012.

\bibitem[Stimper et~al.(2023)Stimper, Liu, Campbell, Berenz, Ryll, Schölkopf, and Hernández-Lobato]{Stimper2023}
Vincent Stimper, David Liu, Andrew Campbell, Vincent Berenz, Lukas Ryll, Bernhard Schölkopf, and José~Miguel Hernández-Lobato.
\newblock normflows: A {P}ytorch package for normalizing flows.
\newblock \emph{Journal of Open Source Software}, 8\penalty0 (86), 2023.

\bibitem[Thompson(1933)]{thompson1933sampling}
William~R Thompson.
\newblock On the likelihood that one unknown probability exceeds another in view of the evidence of two samples.
\newblock \emph{Biometrika}, 25\penalty0 (3-4), 1933.

\bibitem[Tom et~al.(2024)Tom, Schmid, Baird, Cao, Darvish, Hao, Lo, Pablo-García, Rajaonson, Skreta, Yoshikawa, Corapi, Akkoc, Strieth-Kalthoff, Seifrid, and Aspuru-Guzik]{self-driving-lab2024tom}
Gary Tom, Stefan~P. Schmid, Sterling~G. Baird, Yang Cao, Kourosh Darvish, Han Hao, Stanley Lo, Sergio Pablo-García, Ella~M. Rajaonson, Marta Skreta, Naruki Yoshikawa, Samantha Corapi, Gun~Deniz Akkoc, Felix Strieth-Kalthoff, Martin Seifrid, and Alán Aspuru-Guzik.
\newblock Self-driving laboratories for chemistry and materials science.
\newblock \emph{Chemical Reviews}, 124\penalty0 (16), 2024.

\bibitem[Vaswani et~al.(2017)Vaswani, Shazeer, Parmar, Uszkoreit, Jones, Gomez, Kaiser, and Polosukhin]{Vaswani2017AttentionIA}
Ashish Vaswani, Noam Shazeer, Niki Parmar, Jakob Uszkoreit, Llion Jones, Aidan~N Gomez, \L~ukasz Kaiser, and Illia Polosukhin.
\newblock Attention is all you need.
\newblock In \emph{Advances in Neural Information Processing Systems}, 2017.

\bibitem[Wilson et~al.(2018)Wilson, Hutter, and Deisenroth]{Wilson2018Maximizing}
James Wilson, Frank Hutter, and Marc Deisenroth.
\newblock Maximizing acquisition functions for {B}ayesian optimization.
\newblock In \emph{NeurIPS}, 2018.

\bibitem[Wu et~al.(2024)Wu, Kuang, Niu, Ma, and Yu]{wu_diff-bbo_2024}
Dongxia Wu, Nikki~Lijing Kuang, Ruijia Niu, Yian Ma, and Rose Yu.
\newblock Diff-{BBO}: {D}iffusion-based inverse modeling for black-box {O}ptimization.
\newblock In \emph{NeurIPS 2024 Workshop on Bayesian Decision-making and Uncertainty}, 2024.

\bibitem[Xiao et~al.(2024)Xiao, Wang, Zhou, Yuan, Xing, Yan, Wang, Huang, and Liu]{xiao2024omnigen}
Shitao Xiao, Yueze Wang, Junjie Zhou, Huaying Yuan, Xingrun Xing, Ruiran Yan, Shuting Wang, Tiejun Huang, and Zheng Liu.
\newblock Omnigen: Unified image generation.
\newblock \emph{arXiv preprint arXiv:2409.11340}, 2024.

\bibitem[Xie et~al.(2022)Xie, Raghunathan, Liang, and Ma]{Xie2021AnEO}
Sang~Michael Xie, Aditi Raghunathan, Percy Liang, and Tengyu Ma.
\newblock An explanation of in-context learning as implicit {B}ayesian inference.
\newblock In \emph{ICLR}, 2022.

\bibitem[Yun et~al.(2025)Yun, Om, Lee, Yun, and Park]{yun2025posteriorinferencediffusionmodels}
Taeyoung Yun, Kiyoung Om, Jaewoo Lee, Sujin Yun, and Jinkyoo Park.
\newblock Posterior inference with diffusion models for high-dimensional black-box optimization, 2025.

\bibitem[Zaheer et~al.(2017)Zaheer, Kottur, Ravanbakhsh, Poczos, Salakhutdinov, and Smola]{Zaheer2017Deepsets}
Manzil Zaheer, Satwik Kottur, Siamak Ravanbakhsh, Barnabas Poczos, Russ~R Salakhutdinov, and Alexander~J Smola.
\newblock Deep sets.
\newblock In \emph{NeurIPS}, 2017.

\end{thebibliography}
